\pdfoutput=1
\documentclass[sigconf]{acmart}
\usepackage{multirow}
\usepackage{float}
\usepackage[ruled,vlined,linesnumbered]{algorithm2e}
\usepackage{caption}
\usepackage{subcaption}
\usepackage{graphicx}
\usepackage{indentfirst}
\usepackage{tablefootnote}

\usepackage[normalem]{ulem}

\def\alert#1{\textcolor{red}{#1}}
\def\note#1{\textcolor{blue}{#1}}

\usepackage[capitalize,noabbrev]{cleveref}

\usepackage[mode=text]{siunitx}
\newcommand{\SIcommas}[2]{\SI{#1}{#2}} %[group-separator={,},group-minimum-digits={3}]{#1}{#2}}

\DeclareSIUnit\flop{flop}
\DeclareSIUnit\Flop{F}
\DeclareSIUnit\Byte{B}

\newcommand{\mins}[1]{\SIcommas{#1}{\minute}}

\newcommand{\GB}[1]{\SI{#1}{\giga\Byte}}

\newcommand{\numcores}[1]{\SIcommas{#1}{cores}}

\newcommand{\Percent}[1]{\SI{#1}{\percent}}
\AtBeginDocument{%
  \providecommand\BibTeX{{%
    \normalfont B\kern-0.5em{\scshape i\kern-0.25em b}\kern-0.8em\TeX}}}

\begin{document}

\newcommand{\zd}[1]{{\color{blue}{\it DZ: #1}}}
\newcommand{\cy}[1]{{\color{purple}{\it CY: #1}}}
\newcommand{\israt}[1]{{\color{magenta}{\it IN: #1}}}
\newcommand{\rich}[1]{{\color{red}{\bfseries RV: #1}}}
\newcommand{\sys}{XXX\xspace}
\newcommand{\bcomment}[1]{}
\newtheorem{claim}[theorem]{Claim}

%\title{Distributed Tensor-Train Compression for GNN Embeddings}
\title{Nimble GNN Embedding with Tensor-Train Decomposition}
\author{Chunxing Yin}
\affiliation{%
  \institution{Georgia Institute of Technology}
  \country{USA}
}
\email{cyin9@gatech.edu}
\authornote{The majority of this work was done when the author was an intern at Amazon.}

\author{Da Zheng}
\affiliation{%
  \institution{Amazon}
  \country{USA}
}
\email{dzzhen@amazon.com}

\author{Israt Nisa}
\affiliation{%
  \institution{Amazon}
  \country{USA}
}
\email{nisisrat@amazon.com}

\author{Christos Faloutsos}
\affiliation{%
  \institution{Amazon}
  \country{USA}
}
\email{faloutso@amazon.com}

\author{George Karypis}
\affiliation{%
  \institution{Amazon}
  \country{USA}
}
\email{gkarypis@amazon.com}

\author{Richard Vuduc}
\affiliation{%
  \institution{Georgia Institute of Technology}
  \country{USA}
}
\email{richie@cc.gatech.edu}

\begin{abstract}
% @RV: I thought we said we didn't want to describe the method as one for compressing, since we don't really take an existing table and make it smaller; rather, we learn a smaller representation from the beginning. The following edit reflects this notion:
This paper describes a new method for representing embedding tables of graph neural networks (GNNs) more compactly via tensor-train (TT) decomposition.
We consider the scenario where
(a) the graph data that \emph{lack} node features, thereby requiring the learning of embeddings during training;
and
(b) we wish to exploit GPU platforms, where smaller tables are needed to reduce host-to-GPU communication even for large-memory GPUs.
The use of TT enables a compact parameterization of the embedding, rendering it small enough to fit entirely on modern GPUs even for massive graphs.
When combined with judicious schemes for initialization and hierarchical graph partitioning, this approach can reduce the size of node embedding vectors by \num{1659}$\times$ to \num{81 362}$\times$ on large publicly available benchmark datasets, achieving comparable or better accuracy and significant speedups on multi-GPU systems.
In some cases, our model without explicit node features on input can even match the accuracy of models that use node features.
\end{abstract}


\maketitle

\section{Introduction}

\begin{figure}
    \centering
    \includegraphics[width=\columnwidth]{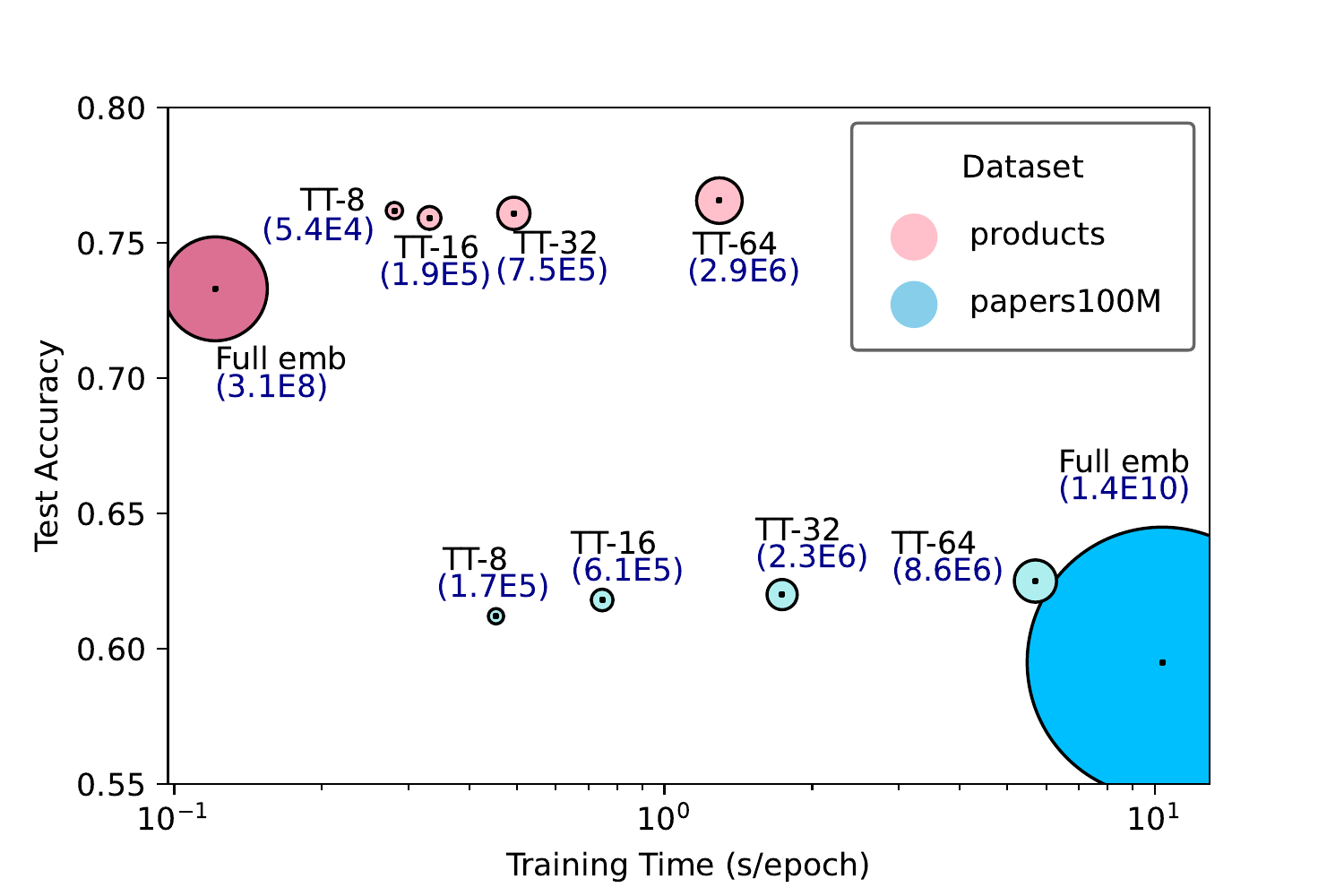}
    \caption{The trade-offs among memory consumption, training time, and accuracy of two OGB node property prediction tasks. The marker area is proportional to the model size, along with the model configuration and number of parameters labeled next to the data points. TT embedding significantly compresses the node embedding in both datasets without sacrificing the model accuracy. Our method accelerates training for large graphs by up to 20$\times$, but introduces overheads to smaller graphs.}
    \label{fig:intro}
\end{figure}

This paper is motivated by a need to improve the scalability of graph neural networks (GNNs), which has become an important part of the modern graph modeling toolkit in an era when graphs with many billions of nodes and edges are commonplace.
For instance, Facebook has billions of users; Amazon has billions of users and billions of items;
the knowledge graph Freebase~\cite{bollacker2008freebase} has 1.9 billion triples and thousands of relations; and Google's Knowledge Graph has over 500 billion facts spanning 5 billion entities.\footnote{\url{https://blog.google/products/search/about-knowledge-graph-and-knowledge-panels/}}

For GNNs, a significant scaling challenge occurs when the graph does not come with explicit node features.
In this case, one typically associates a trainable embedding vector with each node.
The GNN learns these embeddings via message passing, and these trainable node embeddings become part of the overall model's parameterization, with the parameters updated via back propagation.
The problem is the memory requirement for these embeddings, which strains modern GPU-accelerated systems.
Because the embedding storage significantly exceeds the capacity of even large-memory GPUs, a typical strategy is to keep the initial node embeddings on the host CPU and perform mini-batch computations on the GPU~\cite{zheng2020distdgl}.
But in this scenario, severe performance bottlenecks arise due to the cost of doing node embedding updates on the CPU as well as host-to-GPU communication.
%
\begin{comment}
\begin{figure}
    \centering
    \includegraphics[width=\columnwidth]{Figures/intro-new.JPG}
    \caption{The trade-off between node embedding size (x-axis) and accuracy (y-axis) of node property prediction. The colors represent for OGB datasets as follows -- yellow for ogbn-arxiv, blue for ogbn-products, and green for ogbn-papers100M. The circle and square legends represent for GCN(GraphSage) and GAT based models respectively with the embedding table configuration labeled in the figure. We highlight the smallest TT-emb model to outperform the baseline. \cy{Will upload a high quality figure after finalize}}
    \label{fig:intro}
\end{figure}
\end{comment}

There are several methods to accelerate training for models with a large number of trainable embeddings~\cite{weinberger2009feature, serra2017getting,
tito2017hash, kang2020deep} (\cref{sec:related}).
Many use hashing and have been applied mainly in recommendation models.
Recently, a position-based hash method has been developed to compress the trainable node embeddings for GNNs~\cite{kalantzi2021position}.
However, even this method's compression factor of 10$\times$ does not overcome the memory capacity limits of a GPU.

%\zd{maybe we should give some brief introduction on tensor-train to motivate why we pick this method to compress embeddings.}
A recent alternative to reduce embedding storage is the tensor-train (TT) decomposition, which was shown to be effective for recommendation models~\cite{yin2021tt}.
One can think of TT as a compact parameterization of an embedding table that can be much smaller than na\"ively storing the embedding vectors.
Therefore, it is natural to consider this method for GNNs, as we do in this paper.
However, adopting it is nontrivial:
\emph{a priori}, it is unknown how TT will affect accuracy and speed---which TT can actually \emph{worsen}, even if only slightly---as it reduces parameter storage~\cite{yin2021tt}.

For GNNs, we can overcome these risks through additional innovations (\cref{sec:method}).
One innovation is in the TT-parameter initialization scheme.
%\sout{We consider two techniques:
%orthogonal random initialization, which ensures that all columns of the initial embeddings constructed by the tensor-train method are orthogonal;
%and eigenvector initialization, which ensures that the initial embeddings constructed by the tensor-train method are eigenvectors of the adjacency matrix of the input graph.}
%== @RV == Citation for the following claim about depth? And how does orthogonal initialization actually address that problem?
Tensorization implicitly increases the depth of the network and results in slower convergence.
To address this issue, we consider orthogonal random initialization, which ensures that columns of the initial embedding constructed by the tensor-train method are orthogonal.
The other innovation, also motivated directly by the existence of graph structure, is to reorder the nodes in the graph based on a hierarchical graph partition algorithm, namely, METIS~\cite{karypis1997metis}.
This reordering heuristically groups nodes that are topologically close to each other, which leads to a greater ability of homophilic nodes to share parameters in the TT representation.

Combining TT with the novel initialization techniques and node reordering based on graph partition can actually \emph{improve} model accuracy on standard OGB benchmarks~\cite{hu2020ogb} (\cref{sec:results}).
%Both ideas directly exploit the presence of graph data.
We achieved 1\% higher accuracy for the \textsf{ogbn-arxiv} graph, and 3\% for \textsf{ogbn-products} using fewer parameters of \num{415}$\times$ than full embedding table, and 12.2$\times$ than the state-of-the-art position-based embedding~\cite{kalantzi2021position}.
Moreover, our method can train a GNN model on \textsf{ogbn-products} \emph{without} using explicit \textit{a priori} node features to achieve an accuracy comparable to a baseline GNN model that is given such node features.
Meanwhile, because TT can compress the node embeddings by a factor of 424$\times$ to 81362$\times$ on \textsf{ogbn-papers100M}, as we also show, we can fit very large models \emph{entirely on the GPU}.
This capability speeds up training by a factor of 23.9$\times$ and 19.3$\times$ for \textsf{ogbn-papers100M}, which has 111 million nodes, using Graphsage and GAT, respectively.

%
%

%\cy{Should we remove the contribution list?}
% === @RV ===: IMO, yes. It would be more substantive to have a paragraph that summarizes Figure 1, highlighting the tradeoffs/strengths/weaknesses, e.g., smaller and faster but with possibly slightly worse accuracy on papers100M, smaller and more accurate but slower for products. Here is my proposed alternative.
%
The essential highlights and tradeoffs of using our TT-based representation versus a full embedding are summarized in \cref{fig:intro}, which compares accuracy, training time, and the size of the embedding table parameterization for two of the datasets.
    The TT method, which is associated with the so-called TT-rank hyperparameter, produces significantly smaller representations in both datasets, while maintaining comparable accuracy.
    However, training time can speed up significantly (\textsf{ogbn-papers100}) or slow down (\textsf{ogbn-products}).
    These observations both indicate the strong potential of the TT method and raise several questions for future work about how to characterize its efficacy (\cref{sec:conclusion}).
%
\begin{comment}
In summary, our claimed contributions are:
\begin{itemize}
    %\item  This work applies tensor-train compression in GNN. The computation (lookup density) differs from other ML models. We compare our model with the state-of-the-art GNN compression methods on various networks and dataset. \zd{I think our novelty is that when we apply TT method to GNN, we introduce two GNN-related techniques to improve model accuracy.}
    \item This work applies tensor-train compression to GNNs, and introduces two GNN-specific techniques to optimize model accuracy.
    %\item \alert{\sout{We demonstrate that our method only uses the graph structure to achieve model accuracy comparable to GNN models that use node features.}}
    \item \note{We demonstrate the benefit of this technique in 110M node graph training: on papers100M, we compress the model by 81362$\times$ without loss of accuracy and speedup the training by roughly 20$\times$ on single GPU training.}
    \item \note{We propose orthogonal weight initialization algorithms to initialize TT-core weights to improve the model convergence.}
    \item We design a novel method to unify the hierarchical tensor decomposition and graph topology, which introduces weight sharing among vertices that are topologically closer and reveals the graph homophily.
    %incorporate graph partitioning (METIS) into the training, so that the hierarchical tensor decomposition utilizes the graph topology to improve model accuracy.
\end{itemize}
%\zd{i feel there are too many contributions.}
\end{comment}

\section{Related Work}
\label{sec:related}

Among compression techniques for reducing training and inference costs, magnitude pruning~\cite{zhu2017prune}, variational dropout~\cite{molchanov2016pruning}, and $l_0$ regularization~\cite{louizos2017learning} are the best known.
However, these have not been proven to be effective solutions for embedding layer compression.
Thus, we focus our summary on techniques most relevant to embedding-layer compression and tensor decomposition.

\textbf{Embedding Table Compression Techniques:}
Kalantzi et al. explores constructing the node embeddings in GNNs through a composition of a positional-based shared vector and a node specific vector, each from a small pool of embeddings~\cite{kalantzi2021position}.
This work improves accuracy while reducing memory consumption by up to 34$\times$.
However, this level of compression is not enough for the embeddings to fit on a GPU. %, thereby limiting scalability in distributed graph training. 

An alternative is the use of hashing-based methods, especially for personalization and recommendation.
Feature hashing uses a single hash function to map multiple elements to the same embedding vector~\cite{weinberger2009feature}.
However, hash collisions yield intolerable accuracy degradation~\cite{zhao2020distributed}.
The accuracy loss and limited compression ratios appear to be the general issue with unstructured hashing-based methods as well~\cite{shi2020compositional}.
%Guan et al.~\cite{guan2019post} propose a 4-bit  quantization scheme to compress pre-trained models for inference. Although this design is feasible for recommendation inference, quantization for training is more challenging and often comes with accuracy tradeoff.
The application of algebraic methods, such as low-rank approximation on the embedding tables, also suffer unacceptable accuracy losses~\cite{hrinchuk2019tensorized, ghaemmaghami2020training}. 

\textbf{Tensorization:}
\label{sec:related:tensors}
Tensor methods, including Tucker decomposition~\cite{cohen2016expressive}, tensor ring (TR) decomposition~\cite{wang2018wide}, and tensor train (TT) decomposition~\cite{garipov2016ultimate}, have been extensively studied for compressing fully-connected and convolutional layers. 
In particular, TT and TR offer a structural way approximate full tensor weights and thus are shown to be effective in preserving the DNN weights and accuracy.
Tensorized embeddings can produce unique and compact representations, potentially avoiding loss of information caused by mapping uncorrelated items to a same embedding as occurs with hashing.
%and behavior, and so the known techniques to optimize its uncompressed counterpart can usually be effective on the tensor compressed networks with minimum modification.
%TR serves as a generalization of TT and 
%TR can preserve the weights with moderately lower compression ratios than that of TT~\cite{Wang2018}.
%However, there are theoretical and empirical evidence showing that TR is inappropriate for large matrices compression~\cite{batselier2018trouble} because, firstly, a best low-rank TR approximation of a given tensor may not exist. Secondly, the TR-rounding algorithm fails to recover the minimal TR-ranks, and as a result, TR-ranks might grow exponentially after application
%of common operations in neural networks.
%
%\textcolor{blue}{Comparison to Tensor Ring. Someone please delete some details.}
%However the inability to compute an exact minimal rank for TR and the instability severely limits its application to neural networks.
 %Firstly, a best low-rank TR approximation of a given tensor may not exist. Secondly, the TR-rounding algorithm fails to recover the minimal TR-ranks, and as a result, TR-ranks might grow exponentially after application of common operations in neural networks, such as matrix multiplication and tensor contraction. These issues can be solved in TT compression, and thus we choose to compress the embedding tables with TT in this work.
Compressing large embedding layers through TT decomposition has been shown to be effective for recommendation systems~\cite{yin2021tt}, where the model can be more than 100$\times$ smaller with marginal losses of accuracy.
However, this work has not explored the explainability of tensorized networks and can also incur accuracy losses.

\begin{figure*}[t]
    \centering
    \includegraphics[width=.7\textwidth]{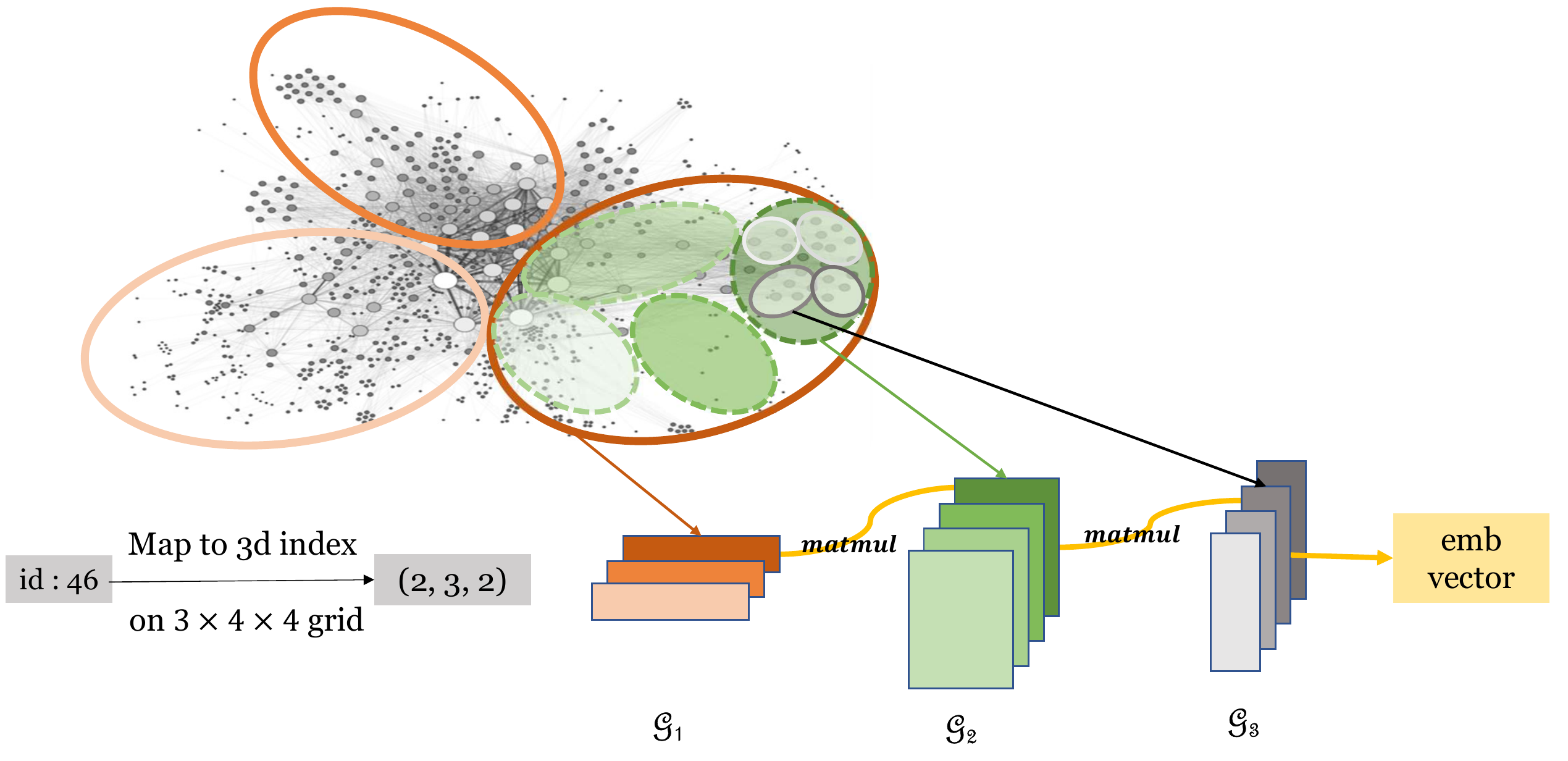}
    \caption{Overview of TT embedding vector construction in GNNs. Embedding vectors are constructed via shared parameters in the corresponding TT cores if the nodes belong to the same cluster in the hierarchical partition.}
    \label{fig:overview}
    %figure source https://yourbasic.org/algorithms/graph/
\end{figure*}

% our analysis, we comprehensively study the design space of how memory size reduction, model quality, and training time overheads trade-off.

%\zd{how about giving some comparison betweeh hash-based embedding compression and tensorization? one advantage i can see from tensorization is that it avoids the collisions between nodes and ensures uniqueness of node embeddings.}

%a high-performance implementation.
%To the best of our knowledge, T3F is the only library that support training TT decomposed networks~\cite{Novikov2018}, and T3nsor %is the only implementation for TT decomposed embeddings~\cite{t3nsor}.
%Both libraries require reforming the full tensor followed by operation using
%the uncompressed tensor during training.
%Using T3F or T3nsor for training can save space for parameter storage, but still requires for the same amount of runtime memory as the uncompressed model, and introduces extra computation only in order to use the compressed the model weights. 

%In \cref{sec:method} we will show our preliminary results on developing efficient algorithms of TT decomposed convolutional layers and embedding tables with asymptotically lower computation cost and memory requirement.

\section{Background}
%\subsection{GNNs and Trainable Embedding}
\paragraph{GNNs and trainable embeddings}
%need rephrase
Graph Neural Networks (GNNs) are deep learning models capable of incorporating topological graph structures and learning node/edge features.
They are often modeled using a message-passing framework, where each node aggregates the messages received from its neighboring nodes.
Given a graph $G(V, E)$ with vertex set $V$ and edges $E$, we denote the \textbf{node embedding} of $v \in V$ at layer $l$ as $h_v^{(l)}$ and the neighborhood of $v$ as $\mathcal{N}(v)$.
A GNN model updates the node embedding via
%Graph Neural Networks (GNNs) are deep learning models that operate on graph structured data. They stack GNN layers to extract topological signals and learn node embeddings. Most of the layers in a GNN model can be expressed under a framework of message passing. Each node sends/receives messages to/from its neighbor nodes. A message is a function of three things: (i) the embedding of the source node, (ii) the embeddings of the destination node and (iii) the edge feature, if available. The message passing framework includes two phrases: message passing and node update. Each node $i$ updates its embedding by receiving messages from its neighbors:
\begin{align}
    h_v^{(l+1)} = g\!\left(h_v^{(l)}, \sum_{u \in \mathcal{N}(v)} f\!\left(h_v^{(l)}, h_u^{(l)}, e_{uv}\right)\right),
\label{eqn:msg-pass}
\end{align}
where $e_{uv}$ is the edge feature associated with the edge $(u, v) \in E$, and $f(\cdot)$ and $g(\cdot)$ are learnable functions that calculate and update the messages, respectively. 
%that maps embeddings of the sender and receiver as well as the corresponding edge features to a message vector, $g(\cdot)$ is a learnable function that updates the node embedding by combining the incoming message and the embedding from the previous layer.

In Equation (\ref{eqn:msg-pass}), $h_v^{(0)}$ represents the input feature of vertex $v$.
If the input features are unavailable or of low-quality, it is common to place a trainable node embedding module in the GNN and use a one-hot encoding for input features.
Let $\mathbf{W} \in \mathbb{R}^{M \times N}$ be the embedding matrix, where $M$ equals the number of vertices and $N$ is the embedding dimension.
The embedding vector $\mathbf{w}_i$ of node $i$ is
\begin{align}
    \mathbf{w}_i^T = \mathbf{e}_i^T \mathbf{W},
\label{eqn:emb-vec}
\end{align}
where $\mathbf{e}_i$ is the one-hot encoding with $i-$th position to be 1, and 0 everywhere else.

%\subsection{Tensor-Train Decomposition}
\paragraph{Tensor-train decompositions}
Let $\mathcal{A}$ be a $d$-dimensional tensor of shape $N_1 \times N_2 \times \dots N_d$.
Its \textbf{tensor train (TT)} is the factorization,
\begin{equation}
    \mathcal{A}(i_1, i_2, \dots, i_d) = \mathcal{G}_1(:, i_1, :)\mathcal{G}_2(:, i_2, :)\dots \mathcal{G}_d(:, i_d, :),
\end{equation}
where each three-dimensional tensor $\mathcal{G}_k$ is a \textbf{TT core}, of shape $R_{k-1}\times N_k \times R_k$.
The values $\{{R_k}\}_{k=0}^d$ are the \textbf{TT ranks}, where $R_0 = R_d = 1$ so that the first and last operands are row and column vectors, respectively.

A TT decomposition can be applied to a matrix, referred to as \textbf{TT matrix format (TTM)}.
We formalize it as follows.
Let $\mathbf{W} \in \mathbb{R}^{M \times N}$ be a matrix and factor each of $M$ and $N$ into any product of integers,
\begin{equation*}
    M = \prod_{k = 1}^d m_k \quad\text{and}\quad N = \prod_{k = 1}^d n_k.
\end{equation*}
Then reshape $\mathbf{W}$ into a $2d$-dimensional tensor of shape $m_1 \times m_2 \times \cdots \times m_d \times n_1 \times n_2 \times \cdots \times n_d$ and shuffle the axes of the tensor to form $\mathcal{W} \in \mathbb{R}^{(m_1\times n_1)\times(m_2\times n_2)\times \dots \times (m_d\times n_d)}$, where
\begin{align}
\begin{split}
    \mathcal{W}&((i_1, j_1),(i_2, j_2), \dots, (i_d, j_d)) \\
    &= \mathcal{G}_1(:, i_1, j_1, :) \mathcal{G}_2(:, i_2, j_2, :) \dots \mathcal{G}_d(:, i_d, j_d, :) 
\end{split}
\label{eqn:tt-matrix}
\end{align}
and each 4-d tensor $\mathcal{G}_k \in \mathbb{R}^{R_{k-1}\times m_k \times n_k \times R_k}$ with $R_0 = R_d = 1$.
Although $\mathcal{G}_1$ and $\mathcal{G}_d$ are 3-d tensors,we will treat them as 4-d tensors for notional consistency, e.g., $\mathcal{G}_1 (0, i_1, j_1, r)$.
\Cref{eqn:tt-matrix} is our compressed representation of the embedding table $\textbf{W}$, which we refer to as \textbf{TT-emb}.

The choice of factorizations of $M$ and $N$ is arbitrary, and tensor reshaping can be performed in the same way with zero padding.
To minimize the number of parameters of the TT-cores, we factor $M$ and $N$ using $m_k = \lceil M^{1/d} \rceil$, and $n_k = \lceil N^{1/d} \rceil$.
TTM format reduces the number of parameters for storing a matrix from $O(MN)$ to $O(dR^2 (MN)^{1/d})$. 
\Cref{tab:ogb-dim-tt} in \cref{sec:results-compress} documents the exact embedding table sizes and our parameter choices.
%%%%%

Given the TT-cores $\{\mathcal{G}_k: \mathcal{G}_k \in \mathbb{R}^{R_{k-1}\times m_k \times n_k \times R_k}\}$, we construct an embedding vector for vertex $i$, which is the $i-$th row of $\textbf{W}$, via
\begin{equation}
    \textbf{w}_i(j) = W(i, j) = \mathcal{G}_1(:, i_1, j_1, :) \mathcal{G}_2(:, i_2, j_2, :) \dots \mathcal{G}_d(:, i_d, j_d, :)
    \label{eqn:tt-embedding}
\end{equation}
where $i = \sum_{p=1}^d i_p \prod_{q = p+1}^d m_q$, and $j = \sum_{p=1}^d j_p \prod_{q = p+1}^d n_q$ for $j = 1, 2, \dots, N$. 
%\cy{Do we need to explain the following computation? Not our contribution.}
An entire row of $\textbf{W}$ can be reconstructed all at once (rather than elementwise) using $d-1$ matrix multiplications~\cite{yin2021tt}.
%
\begin{comment}
    \cy{Can we remove the following details?}
\textcolor{brown}{Each 3-d subtensor of TT-core $\mathcal{G}_k(:, i_k, :, :)$ is unfolded into a matrix of shape $R_{k-1} \times n_kR_k$.
Let $\textbf{w}_i^{(k)} \in \mathbf{R}^{\prod_1^{k-1} n_i \times n_kR_k}$ be the partial product of the first $k$ unfolding matrices. We reshape $\textbf{w}_i^{(k)}$ to be of shape $\prod_1^{k} n_i \times R_k$ to proceed this chain of matrix multiplications, which in the end calculates the entire row $w_i^{(d)} \in \mathbf{R}^{N \times 1}$.}
\end{comment}
%
Unfold each 3-d subtensor of $\mathcal{G}_k(:, i_k, :, :)$ into a matrix of shape $R_{k-1} \times n_k R_k$.
Let $\textbf{w}_i^{(k)} \in \mathbf{R}^{\prod_1^{k-1} n_i \times n_kR_k}$ be the partial product of the first $k$ unfolding matrices.
We reshape $\textbf{w}_i^{(k)}$ to be of shape $\prod_1^{k} n_i \times R_k$ and proceed with this chain of matrix multiplications, which in the end calculates the entire row $w_i^{(d)} \in \mathbf{R}^{N \times 1}$.

\begin{table}
\caption{TT-emb related Notations.}
\label{tab:notation}
\begin{tabular}{ll}
\toprule
Notation                                 & Description  \\ 
\hline \rule{0pt}{2.6ex}
$\mathbf{W} \in \mathbb{R}^{M \times N}$          & $M \times N$ trainable embedding table \\
$\mathcal{W}$                             & $\mathbf{W}$ reshaped as a $2d$-way tensor \\
$m_i, n_i$                               & the $i$-th factor of $M$ and $N$, resp. \\
$d$                                      & number of TT cores \\
$R_i \in \mathbb{R}$                       & predefined TT ranks \\
$\mathcal{G}_i \in \mathbb{R}^{R_{i-1} \times M_i \times N_i \times R_k}$ & TT-core for matrix decomposition \\
\bottomrule
\end{tabular}
\end{table}

This pipeline for embedding lookup with TTM is illustrated in \cref{fig:overview}.
As an example, assume the graph has $N = 48$ nodes, and the node with id 46 is queried. First, we map the node id to a 3-d coordinate $(i_1, i_2, i_3) \in \mathbb{Z}_+^{ m_1 \times m_2 \times m_3}$, namely $(2, 3, 2)$ in this case. To reconstruct the embedding vector for node $46$, we multiply the $i_k^{th}$ slice from TT-core $\mathcal{G}_i$.

%The tensor multiplication of $w_i^{(k)}\mathcal{G}_{k+1}(:, i_k, :, :)$ can be unfolded and formulated as a matrix-matrix multiplication where $w_i^{(k)} \in \mathbf{R}^{\prod_1^k n_i \times R_k}$ and $\mathcal{G}_{k+1}(:, i_k, :, :) \in \mathbf{R}^{R_k\times n_{k+1}R_{k+1}}$.

\section{Methods}
\label{sec:method}

%\subsection{Overview}
%\label{sec:method-overview}

Our method is motivated by two principles.
First, we wish to learn a \emph{unique} embedding vector for each node.
Doing so intends to reduce information loss or accuracy degradation that might arise when mapping distinct nodes into the same vector, as with hashing (\cref{sec:related}).
Second, we seek node embeddings that reflect \emph{graph homophily}~\cite{mcpherson2001homophily}.
That is the property where topologically close nodes tend to share similar representations and labels, which might result in similar node embeddings.

The method itself has two major components.
We apply TT decomposition to reduce the embedding table size in GNNs and leverage graph homophily by aligning the graph's topology with the hierarchy of the tensor decomposition.
To promote convergence, we propose orthogonal weight initialization algorithms for TT-emb in GNNs (\cref{sec:method-ortho}).
We then incorporate the graph topological information into tensor decomposition by preprocessing the graph through recursive partitioning and rearranging nodes by clusters, which would enforce more similar embedding construction among neighboring nodes.
Our use of graph partitioning tries to maximize TT-core weight sharing in the embedding construction (\cref{sec:method-metis}).

\subsection{Orthogonal Initialization of TT-cores}
\label{sec:method-ortho}
%==================================================
%\sout{The eigenvectors of a symmetric matrix, such as the adjacency matrix of unweighted bidirectional graphs, forms an orthonormal basis. While it is computationally expensive to calculate the eigenvectors of a large matrix (in the order of millions or billions), an orthonormal matrix of the same dimension can be efficiently generated. We study the relaxation of eigenvector initialization to orthogonal initialization, and present an algorithm to initialize the tensorized embedding.}

Previous studies have suggested thats, comparing to Gaussian initialization, initializing network weights from the orthogonal group can accelerate convergence and overcome the dependence between network width and the width needed for efficient convergence to a global minimum~\cite{hu2020provable}.
%Therefore, our algorithm for TT-core initialization ensures that the unfolding matrix of each $\mathcal{G}_i$, as well as the product embedding matrix $\mathbf{W}$, is orthogonal.
Therefore, our algorithm for TT-core initialization ensures that the product embedding matrix $\mathbf{W}$ is orthogonal.

For simplicity, we describe and analyze the algorithm assuming a 3-d TT decomposition, which can be easily extended to higher dimensions.
Let $\textbf{W} \in \mathbb{R}^{M \times N}$ be a matrix which we decompose as
\begin{align*}
    \mathcal{W}((i_1, j_1&), (i_2, j_2), (i_3, j_3)) 
    \\&= \mathcal{G}_1(:, i_1, j_1,:)\mathcal{G}_2(:, i_2, j_2, :) \mathcal{G}_3 (:, i_3, j_3,:),
\end{align*}
where $\mathcal{G}_1\in \mathbb{R}^{1 \times m_1 \times n_1 \times R}$, $\mathcal{G}_2\in \mathbb{R}^{R \times m_2 \times n_2 \times R}$, and $\mathcal{G}_3\in \mathbb{R}^{R \times m_3 \times n_3 \times 1}$. 

One intuitive approach is to decompose a random orthogonal matrix $\mathbf{W} \in \mathbb{R}^{M \times N}$. 
\Cref{alg:ttm} describes the process of decomposing a matrix into its TTM format, adapted from the TT-SVD decomposition algorithm~\cite{oseledets2011tensor}.
This method ensures the orthogonality of the constructed embedding matrix but not of each single TT-core.
%This relaxed condition on orthogonality has little or no accuracy impact in practice, as \cref{sec:result-init} shows.

\begin{algorithm}
    \caption{TTM Decomposition} 
    \label{alg:ttm}
    \KwIn{Embedding matrix $\textbf{X} \in \mathbb{R}^{M \times N}$ consists of normalized eigenvectors of $L$}
    \KwOut{TT-cores $\mathcal{G}_1, \dots, \mathcal{G}_d$ of $X$ as defined in Eqn~\ref{eqn:tt-matrix}}
    $\mathcal{X}$ = reshape($X$, [$m_1, m_2 \dots m_d, n_1, n_2, \dots, n_d$])\;
    $Y$ = transpose($\mathcal{X}$, [$m_1, n_1, m_2, n_2 \dots, m_d, n_d$])\;
    \For{k = 1 to d-1}{
        $Y$ = reshape($Y$, $[R_k * m_k * N_k, :]$)\;
        Compute truncated SVD $Y = U \Sigma V^T + E$, $rank(U) \le R_k$\;
        $\mathcal{G}_k$ = reshape($U$, $[R_k, m_k, n_k, R_{k+1}]$)\;
        $Y$ = $\Sigma V^T$\;
    }
    $\mathcal{G}_d$ = reshape($Y$, [$R_d, m_d, n_d, 1$])\;
\end{algorithm}
However, constructing a large orthogonal matrix and its TTM decomposition can be slow. We show an efficient algorithm to initialize unfolding matrices of each $\mathcal{G}_i$ from small orthogonal matrices, in a way that ensures the orthogonality of the product embedding matrix $\mathbf{W}$.
Formally, we want to initialize the tensor cores so that $\textbf{W}$ is uniformly random from the space of scaled semi-orthogonal matrix, i.e., those having $\textbf{W}^T \textbf{W} = \alpha \textbf{I}$.
Our algorithm iteratively initializes each TT-core as shown in \cref{alg:ortho}.

%\vspace{-0.4cm}
\bcomment{
\begin{algorithm}
    \caption{Orthogonal Initialization of TT-cores.}
    \label{alg:ortho}
    Generate $Rn_3$ orthogonal vectors of length $m_3$ $\{v_1, v_2, \dots, v_{Rn_3}\}$\;
    Set $\mathcal{G}_3(r, :, j, 0)= v_{rn_3+j}$\;
    Generate $Rn_2$ orthogonal vectors of length $Rm_2$ $\{w_1, w_2, \dots, w_{Rn_2}\}$\;
    Set $\mathcal{G}_2(r, : , j, :) = w_{rn_2+ j}.reshape(m_2, R)$\;
    Generate $n_1$ orthogonal vectors of length $Rm_1$ $\{u_1, u_2, \dots, u_{n_1}\}$\;
    Set $\mathcal{G}_1(0, : , j, :) = u_{j}.reshape(m_1, R)$\;
\end{algorithm}
}
\begin{algorithm}
    \caption{Orthogonal Initialization of TT-cores.}
    \label{alg:ortho}
    \KwIn{Matrix shapes $M, N$ factorized as $M = \prod^d_1 m_i$, $N = \prod_1^d n_i$, and TT-ranks $\{R_0, R_1, \dots, R_d\}$}
    \KwOut{TT-cores $\{\mathcal{G}_i\}_{i=1}^d$ s.t. the unfolding matrix of $\mathcal{G}_i$ and product $W \in \mathbb{R}^{M \times N}$ are semi-orthogonal}
    \For{i = 1 to d}{
        Generate $n_i R_{i-1}$ orthogonal vectors $\{v_1, v_2, \dots v_{n_i R_{i-1}}\}$, where $v_j \in \mathbb{R}^{m_i R_{i}}$ \;
        \For{r = 1 to $R_{i-1}$}{
            \For{j = 1 to $n_i$}{
                $\mathcal{G}_i(r, :, j, :) = $ reshape$(v_{n_i r + j}, [m_i, R_{i}])$\;
            }
        }
    }
\end{algorithm}

\begin{claim}
The embedding matrix $W$ initialized by Algorithm~\ref{alg:ortho} is column orthogonal. 
\label{claim:ortho}
\end{claim}
\begin{proof}
We denote the factorization of a row index $0 \le p \le M$ as $p = \sum_{\alpha = 1}^3 p_\alpha m_\alpha$, and the factorization of a column index $0 \le q \le N$ as $j = \sum_{\alpha = 1}^3 q_\alpha n_\alpha$. To show the orthogonality of $\textbf{W}$, we compute the entry $(i,j)$ in $\textbf{W}^T \textbf{W}$
\begin{align}
\begin{split}
    &\textbf{W}^T \textbf{W}(i,j) = \sum_{k = 1}^M \textbf{W}^T(i, k) \textbf{W}(k, j) = \sum_{k=1}^M \textbf{W}(k,i) \textbf{W}(k,j)\\
    %& = \sum_{k = 1}^n \mathcal{G}_1(:, k_1, i_1,:)\mathcal{G}_2(:, k_2, i_2, :) \mathcal{G}_3 (:, k_3, i_3,:) \\
    %&(\mathcal{G}_1(:, k_1, j_1,:)\mathcal{G}_2(:, k_2, j_2, :) \mathcal{G}_3 (:, k_3, j_3,:))\\
    %& = \sum_{k = 1}^n \mathcal{G}_1( k_1, i_1,:)\mathcal{G}_2(:, k_2, i_2, :) \mathcal{G}_3 (:, k_3, i_3) \\
    %&(\mathcal{G}_3 (:, k_3, j_3)^T\mathcal{G}_2(:, k_2, j_2, :)^T\mathcal{G}_1(k_1, j_1,:)^T  )\\
    & = \sum_{k_1, k_2, k_3} \mathcal{G}_1(0, k_1, i_1,:)\mathcal{G}_2(:, k_2, i_2, :) (\mathcal{G}_3 (:, k_3, i_3, 0) \mathcal{G}_3 (:, k_3, j_3, 0)^T)\\
    &\mathcal{G}_2(:, k_2, j_2, :)^T\mathcal{G}_1(0, k_1, j_1,:)^T
    .
\end{split}
\label{eqn:semi-ot}
\end{align}
We initialize $\mathcal{G}_3$ as line 1-2 in Algorithm~\ref{alg:ortho}.
We show that $\textbf{A} = \sum_{k_3}\mathcal{G}_3 (:, k_3, i_3,0) \mathcal{G}_3 (:, k_3, j_3,0)^T$ is orthogonal.
If $i_3 = j_3$,
\begin{align}
    &\textbf{W}(p, q) = 
    \begin{cases}
        \sum_{k_3} \mathcal{G}_3(p, k_3, i_3, 0)^2 = 1 & \mbox{, if } p = q,\\
        \sum_{k_3} \mathcal{G}_3(p, k_3, i_3, 0) \mathcal{G}_3(q, k_3, i_3, 0) = 0& \mbox{, if } p\neq q .
    \end{cases}
\label{eqn:G3}
\end{align}
Similarly, if $i_3 \neq j_3$, we can verify that $A(p,q)$ is always 0.
Therefore, 
\begin{align}
     \textbf{A} = &\sum_{k_3}\mathcal{G}_3 (:, k_3, i_3, 0) \mathcal{G}_3 (:, k_3, j_3, 0)^T = 
    \begin{cases}
        \textbf{I} , & \mbox{ if } i_3 = j_3\\
        0, &\mbox{ if } i_3 \neq j_3
    \end{cases}
\end{align}
If $i_3 \neq j_3$, then \cref{eqn:semi-ot} is 0. Otherwise, $i_3 = j_3$ and
\begin{align*}
    &\textbf{W}^T \textbf{W}(i,j) \\
    &=\sum_{k_1, k_2} \mathcal{G}_1( 0,k_1, i_1,:)\mathcal{G}_2(:, k_2, i_2, :)
    \mathcal{G}_2(:, k_2, j_2, :)^T \mathcal{G}_1(0,k_1, j_1,:)^T.
\end{align*}
%if $i_3 = j_3$, and 0 otherwise.

Similarly, we can show that lines 3-6 of \cref{alg:ortho} yield
\begin{align*}
    \sum_{k}\mathcal{G}_l(:, k, i_l, :) \mathcal{G}_l(:, k, j_l, :)^T = \begin{cases}
    \alpha \textbf{I} &\mbox{, if } i_l = j_l\\
    0 &\mbox{, if } i_l \neq j_l
    \end{cases},
    \text{ for } l = 1, 2
\end{align*}
This completes the proof that $\textbf{W}$ is column orthogonal.
\end{proof}

The shape of the embedding matrix will satisfy $m \gg n$, ensuring that $R_{i-1}n_i < R_im_i$ for $i = 2, 3$.
Therefore, it is always feasible to initialize the tensor cores with $Rn_\alpha$ orthogonal vectors.
But it is not guaranteed that $R_2n_3 \ge m_3$ in $\mathcal{G}_3$.
For instance, suppose $m_3 = \lceil m^{1/3} \rceil = 126$ and $n = \lceil n^{1/3} \rceil = 5$ when $m = 10^8$ and $n = 128$.
Then $Rn_3 \ge m_3$ is only satisfied when $R \le 25$.
Thus, we prefer to order the factors of the matrix dimensions $m, n$ such that $m_3 = \max(\text{factor}(m))$ and $n_3 = \min(\text{factor}(n))$. 
Although this algorithm might be infeasible for large TT-ranks, we show in Section~\ref{sec:results} that small TT-ranks are sufficient for many real-world graphs.

%==================================================
\subsection{Homophily-informed TT Alignment}
\label{sec:method-metis}
%==================================================
%\zd{we need to add some discussion here to motivate why ordering nodes are important for tensor-train. basically, we need to jump back to the tensor-train algorithm and show how the nodes share parameters in the tensor-train methods. we should explain why this is a problem and why we like to introduce a hierarchical graph partition algorithm in our method to share parameters between nodes differently.}

%In many real-world graphs, nodes that are topologically close tend to share the same representations and labels. To exploit graph \cy{1) homophily}, we propose to utilize hierarchical graph partitioning to ensure neighboring nodes are more likely to have similar node embeddings. Moreover, the designed embedding should construct a \cy{2) unique} vector for each node, to avoid loss of information or accuracy degradation from  mapping distinct nodes to a same vector.

The embedding similarity between neighboring nodes can be inherited from TT format given an appropriate node ordering.
%To see this, 
%Therefore the neighboring nodes share the same subset of TT-core weights when constructing the embedding vector.
Let $i$ be the index of a node.
Its embedding vector is constructed from
\begin{align*}
    W[i, :] = \mathcal{G}_1(0,i_1, :, :)\mathcal{G}_2(:,i_2, :, :)\mathcal{G}_3(:,i_3, :, 0),
\end{align*}
where
\begin{align*}
    &i_1 = i / (m_2 * m_3), 
    i_2 = (i / m_3) \text{ mod } m_2 \text{, and }\\
    &i_3 = i \text{ mod } m_3.
\end{align*}
\begin{figure}[H]
    \centering
    \includegraphics[width=0.8\columnwidth]{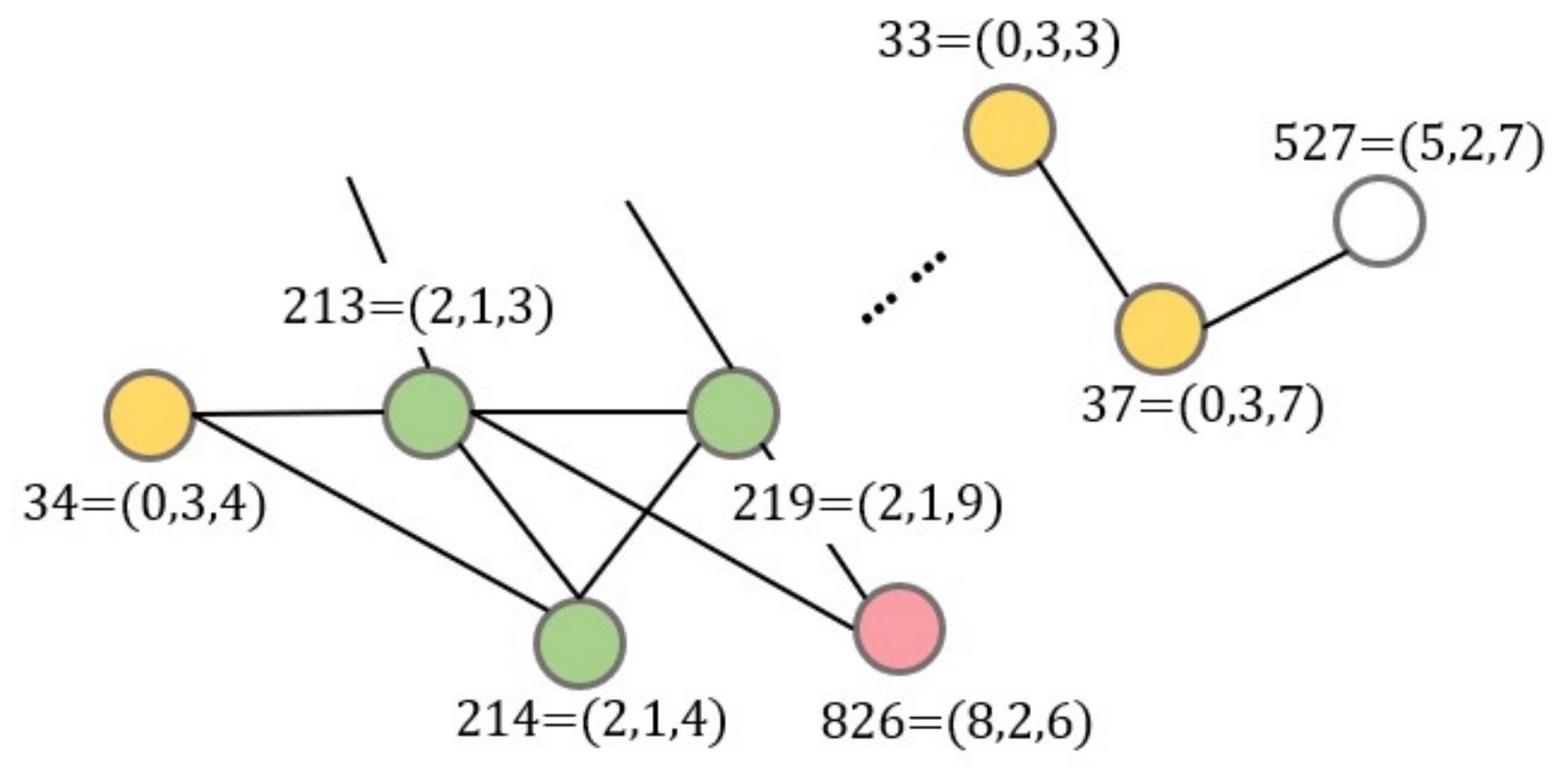}
    \caption{3D coordinate mapping of an example graph with \num{1000} nodes, assuming $m_1 = m_2 = m_3 = 10$.
    Nodes of the same color share the same $\mathcal{G}_1$ and $\mathcal{G}_2$ subtensors during embedding construction.}
    \label{fig:node-index}
\end{figure}

To see how weight sharing among nodes is determined by node ids, consider the illustration in \cref{fig:node-index}.
If the nodes are randomly ordered, constructing an embedding vector is equivalent to drawing one slice $\mathcal{G}_l(:,i_l, :, :)$ uniformly at random for $l = 1, 2, 3$ and composing the weights.
Consequently, distant vertices may share weights.
Instead, ordering the nodes based on graph topology would better promote weight sharing among topologically neighboring nodes and thereby produce homophilic representations.

Formally, for any node index $j \in [i - i_3, i+m_3 - i_3]$, the row $\textbf{W}[j, :]$ is constructed using the shared weights $\mathcal{G}_1(0,i_1, :, :)\mathcal{G}_2(:,i_2, :, :) \in \mathbb{R}^{m_1m_2 \times r}$.
The two rows $\textbf{W}[i, :]$ and $\textbf{W}[j, :]$ are formed by the same linear transformation of different set of vectors $\mathcal{G}_3(:,i_3, :, 0)$ and $\mathcal{G}_3(:,j_3, :, 0)$ respectively.
Similarly, all the nodes with indices in range $[i - i_3 - i_2m_3, i+m_3 - i_3 + (m_2 - i_2)m_3]$ share at least one factor matrix $\mathcal{G}_1(0,i_1, :, :)$ in embedding vector construction.
%\zd{we should first say that if node IDs are labeled randomly, this introduce unnecessary weight sharing. However, if nodes are ordered based on graph topology, this weight sharing strategy aligns with the graph homophily.}
The graph's topology influences the weight sharing and embedding vectors can be more similar when the node indices are closer.
Hence, we need to preprocess the graph in a way that the topological closeness between nodes is proportional to the difference of their indices.

We perform hierarchical graph partitioning to discover clusters and reorder the nodes accordingly.
We denote the original graph as level-0;
to form the level $i+1$ partitioning, we recursively apply a $p_i-$way graph partitioning on each of the communities at level $i$.
In this way, there will be $p_1$ partitions at level 1, $p_1 p_2$ partitions at level 2, and so on so forth.
Let $L$ be the total number of partition levels and $0 \le p \le p_1 p_2 \dots p_L$ be a partition id at the $L$-th level partition.
We reassign the indices to nodes in the partition from range $[np/(p_1 p_2 \dots p_L), n(p+1)/ (p_1 p_2 \dots p_L)]$, where $n$ is the number of nodes in the graph.
We will assess this approach in \cref{sec:result-metis}. % that graph partitioning and node reordering can significantly improve the model accuracy.
%\zd{the reference is incorrect here. how about summarizing the performance number here directly?}

\section{Experimental Setup}

\subsection{Datasets}
\label{sec:exp:datasets}
We use three node-property prediction datasets from Open Graph Benchmark~\cite{hu2020ogb}, summarized in \cref{tab:dataset}.
We add reversed edges to the \textsf{ogbn-arxiv} and \textsf{ogbn-papers100M} datasets to make the associated graphs undirected.
% i don't think you need to say it.
%The data splits are default as provided by OGB.

We train all learnable embedding models using one-hot encoding and without the original node features provided by OGB.
We also train each model using the original node features given by OGB as a baseline to understand the quality of our design.
The embedding dimensions of our embeddings are set to be the same as the datasets' original node features.

\begin{table}[h]
\caption{Dataset statistics.
We add reversed edges to \textsf{ogbn-arxiv} and \textsf{ogbn-papers100M} to make the graph undirected.}
\label{tab:dataset}
\begin{tabular}{lrrcc}
\toprule
Dataset                  & \#Nodes           & \#Edges             &   Emb. Dim. \\ \hline
\textsf{ogbn-arxiv}      & \num{169 343}     & \num{2 332 486}     & \num{128}   \\
\textsf{ogbn-products}   & \num{2 449 029}   & \num{61 859 140}    & \num{100}   \\
\textsf{ogbn-papers100M} & \num{111 059 956} & \num{3 231 371 744} & \num{128}   \\
\bottomrule
\end{tabular}
\end{table}

\subsection{Hardware Platforms}
The machine configurations we used are summarized in \cref{tab:hardware}.
We train all the models using the AWS EC2 \textsf{g4dn.16xlarge} instance and evaluate the training efficiency on EC2 multi-GPU instances \textsf{p3.16xlarge} and \textsf{g4dn.metal}.
The \textsf{r5dn.24xlarge} instance is only used for preprocessing ogbn-papers100M dataset.

\begin{table}
\caption{Hardware platform details.}
\label{tab:hardware}
\begin{tabular}{ll}
\toprule
EC2 Type      & Hardware specification  \\ \hline
\textsf{r5dn.24xlarge}  & 2$\times$\numcores{24}, \GB{700} RAM, 100 Gbit/s network               \\
\textsf{p3.16xlarge}    & 2$\times$\numcores{16}, \GB{500} RAM, 8 V100 GPUs               \\
\textsf{g4dn.metal}     & 2$\times$\numcores{24}, \GB{384} RAM, 8 T4 GPUs                \\ 
\textsf{g4dn.16xlarge}  & 2$\times$\numcores{16}, \GB{256} RAM, 1 T4 GPUs                \\\bottomrule
\end{tabular}
\end{table}

\subsection{GNN Models and Baseline}
For each dataset, we choose two different GNN models implemented with Deep Graph Library (DGL)~\cite{wang2019deep}.
For \textsf{ogbn-arxiv}, we use GCN~\cite{kipf2016semi} and GAT~\cite{velivckovic2017graph} models.
For the two larger datasets, we use GraphSage~\cite{hamilton2017inductive} and GAT.
The hyperparameters for each model are set to those tuned by the DGL team. 

We consider the full embedding table models (Full-Emb) and position-based hash embedding (Pos-Emb)~\cite{kalantzi2021position} as our baselines.
All models are trained using one-hot encoding and no additional node features for each node, so that the model learns a unique embedding vector per node as the input to the GNNs.

\section{Results}
\label{sec:results}
%\zd{I think the evaluation should start with model accuracy. you can provide model accuracy with different configurations. Then you should model compression ratio and discuss the scalability of this method. Then you should discuss the training efficiency and speedup. In the end, you should have an ablation study to show the effectiveness of the two optimizations to show that both optimizations are important.}
%
We evaluate the performance of our proposed algorithms, varying the weight initialization, hierarchical graph partition, and training cost. 
We use \textit{TT-Emb(r)} to refer to the GNN model with a trainable tensorized embedding with TT rank $r$.
We compare our models to three other baseline options:
\begin{itemize}
    \item \textit{Orig-Feat}: Use the original features given by the OGB dataset as the non-trainable node feature.
    The features are preprocessed text associated with the nodes, such as paper abstract and product description.
    \item \textit{Full-Emb}: Train a full embedding table only with one-hot encoding for each node.
    \item \textit{Pos-Emb}: A positional-based hash embedding trained with one-hot encoding for each node.
\end{itemize}

\begin{comment}
\begin{table}[H]
\caption{Performance Summary}
\label{tab:acc-sum}
\resizebox{\columnwidth}{!}{%
\begin{tabular}{@{}lllllll@{}}
\toprule
            & \multicolumn{2}{c}{\textbf{ogbn-arxiv}} & \multicolumn{2}{c}{\textbf{ogbn-products}} & \multicolumn{2}{c}{\textbf{ogbn-papers100M}} \\
Method      & GCN                & GAT                & Graphsage               & GAT              & GraphSage                & GAT               \\ \midrule
Full-Emb     & 0.671              & 0.677              &0.733                    & 0.755            & 0.595             & N/A$^+$         \\
Orig-Feats  & 0.724              & 0.736              & 0.782                   & 0.786            & 0.655                & 0.655             \\
Pos-Emb     & 0.674              & 0.676              & 0.758                   & 0.767            &  N/A                     &  N/A              \\
\textbf{TT-Emb}      & \textbf{0.681}              & \textbf{0.682}              & \textbf{0.762}          & \textbf{0.790}   &  \textbf{0.595}                   &  0.647             \\
Compression* & 22.9$\times$       &   1117$\times$     & 5762$\times$            & 415$\times$      &  424$\times$             &  424$\times$     \\ \bottomrule
\end{tabular}
}
\raggedright
\footnotesize{$*$ Compression ratio of TT-emb compared to FullEmb and Orig-Feats. \\
$+$ Out of memory on the experiment platform.}
\end{table}
\end{comment}

\begin{table}[H]
\caption{Performance summary of smaller graphs.}
\label{tab:summary-small}
\resizebox{\columnwidth}{!}{%
\begin{tabular}{lllllllll}
\toprule
                                 & \multicolumn{4}{c}{\textbf{ogbn-arxiv}}                                                         & \multicolumn{4}{c}{\textbf{ogbn-products}}                                                      \\
Method                           & GCN    & GAT      & CR*          & Time$^+$ & Graphsage      & GAT           & CR* & Time$^+$ \\ \midrule
Full-Emb                         & 0.671  & 0.677    & 1$\times$    & 0.49    & 0.733          & 0.755         & 1$\times$    & 33.4     \\
Orig-Feats                       & 0.724  & 0.736    & 1$\times$    & 0.41    & 0.782          & 0.786         & 1$\times$    & 31.5     \\
Pos-Emb                          & 0.674  & 0.676    & 12$\times$   & -         & 0.758          & 0.767         & 34$\times$   &  -    \\
\textbf{TT-Emb}    & \textbf{0.681} & \textbf{0.682} & \num{1117}$\times$ & 0.97     & \textbf{0.761} & \textbf{0.784}& 415$\times$  &  75.8    \\
 \bottomrule
\end{tabular}%
}
\raggedright
\footnotesize{$*$ Compression ratio of embedding layer compared to FullEmb and Orig-Feats.\\
$+$ Embedding layer training time (sec) per epoch for ogbn-arxiv with GCN and ogbn-products with GraphSage on g4dn.16xlarge instance.}
\end{table}

\begin{table}[H]
\caption{Accuracy, training time per batch, and compression ratio (CR) of \textsf{ogbn-papers100M} using different TT ranks. }
\label{tab:paper-accu}
\resizebox{\columnwidth}{!}{%
\begin{tabular}{lccccl}
\toprule
           & \multicolumn{2}{c}{\textbf{Graphsage}} & \multicolumn{2}{c}{\textbf{GAT}} & \textbf{CR}    \\
           & Accuracy        & Time        & Accuracy     & Time     &       \\ \midrule
FullEmb    & 0.595           & 10.34       & N/A$^*$      & 7.16     & -     \\
TT-Emb(8)  & 0.612           & 0.39        & 0.632        & 0.37     & \num{81 362} \\
TT-Emb(16) & 0.618           & 0.45        & 0.631        & 0.45     & \num{23 479} \\
TT-Emb(32) & 0.620           & 0.51        & 0.633        & 0.68     & \num{6 360}  \\
TT-Emb(64) & 0.625           & 1.74        & 0.634        & 1.67     & \num{1659}  \\
\bottomrule
\end{tabular}%
}
\raggedright
\footnotesize{$*$ Out of memory on the experiment platform.\\
$+$ The graph is preprocessed with \num{6400} partitions. TT-emb(64) is initialized using \cref{alg:ttm}, and others using \cref{alg:ortho}.}
\end{table}

\Cref{tab:summary-small} summarizes the accuracy, compression ratio, and training time of the baseline models and our optimal configuration for the two smaller datasets. \cref{tab:paper-accu} shows the performance on \textsf{ogbn-papers100M} with detailed configuration of TT-emb.
For each dataset, we need to choose the TT rank, initialization algorithm, and the graph partitioning granularity. Our design is able to achieve the best accuracy among all models trained with one-hot encoding with order of magnitude fewer parameters.
In particular, on \textsf{ogbn-products}, the quality of our compact embedding is comparable to the original node feature.

We study the accuracy of tensorized models initialized from different algorithms in \cref{sec:result-init}.
Then in \cref{sec:result-metis}, we analyze how graph partitioning affects the performance of the model.
We discuss the compression and scalability of our model in \cref{sec:results-compress}. % and Section~\ref{sec:result-time}.

\subsection{Weight Initialization}
\label{sec:result-init}
\bcomment{
\begin{table}[]
\caption{Accuracy results of orthogonal and eigenvector initialization for GraphSage and GAT trained on ogbn-products without graph partition.}
\label{tab:accu-init}
\begin{tabular}{clcc}
\hline
  Model                      & Initialization & GraphSage & GAT \\ \hline
Full-Emb           & Gaussian       & 0.7338    & 0.755    \\ 
Orig-Feats            & -              & 0.7832    & 0.7622    \\ 
Pos-Emb                      & -              & 0.760     &  0.762   \\\hline
\multirow{3}{*}{TT-Emb(8)}   & eigen      & 0.7010    &  0.7457   \\
                             & orthogonal & 0.6749    &  0.7426   \\
                             & Gaussian   & 0.6682    &  0.7395   \\ \hline
\multirow{3}{*}{TT-Emb(16)}  & eigen      & 0.7232    &  0.7544   \\
                             & orthogonal & 0.7131    &  0.7464   \\
                             & Gaussian   & 0.7147    &  0.7421   \\ \hline
\multirow{3}{*}{TT-Emb(32)}  & eigen      & 0.7217    &  0.7606   \\
                             & orthogonal & -         &  -   \\
                             & Gaussian   & 0.7057    &  0.7476   \\ \hline
\end{tabular}
\end{table}
}
We assess the impact of TT ranks and initialization algorithms from \cref{sec:method-ortho} on the model accuracy. 
Taking \textsf{ogbn-products} as an example, \cref{tab:prod-accu} summarizes the accuracy of models trained while varying the three most important hyperparameters: TT rank, initialization algorithm, and graph partitioning granularity.
We focus on TT ranks and weight initialization in this section and leave graph partitioning to \cref{sec:result-metis}. 

The three horizontal groups show the model accuracy as TT rank ranges from \num{8} to \num{32}.
Regardless of the initialization algorithm, \Cref{tab:prod-accu} shows that, for GraphSage and GAT, larger TT ranks generally yield higher model accuracies.
Increasing TT rank significantly improves model performance, especially with a coarse partitioning.
%We notice that higher TT rank should provide a more accurate mathematical approximation to the full embedding model, however empirically, increasing TT rank beyond a certain point makes the model more likely to overfit and thus degradate the test accuracy. 

The three rows under each \textit{TT-Emb(r)} group in \cref{tab:prod-accu} show that the accuracy of the models differ only by initialization algorithm. We refer to \cref{alg:ttm} as \textsf{decomp-ortho} and \cref{alg:ortho} as \textsf{ortho-core} in the table.
Comparing to initializing TT-Emb from Gaussian distribution, the two orthogonal initialization algorithms improve the test accuracy in a similar way across all configurations and models.
However, the decomposed orthogonal initialization (\cref{alg:ttm}) comes with a much higher computational cost.
It takes \mins{31} for \textsf{ogbn-papers100M} on a \textsf{r5dn.24xlarge} instance, where \Percent{25} of the time is spent on generating a random orthogonal matrix of size $110M \times 128$ using Scipy's QR decomposition~\cite{2020SciPy-NMeth}.
By contrast, \cref{alg:ortho} only takes a few seconds to generate orthogonal TT-cores using the Gram-Schmidt method for a graph of the same size.

In summary, for smaller TT ranks, the orthogonal TT-core initialization provides higher efficiency;
for large TT ranks, the TTM decomposition of orthogonal matrices yields better scalability.

\begin{comment}
\note{%
Figure~\ref{fig:rank-vs-acc} illustrates the correlation between the effective rank introduced by~\cite{roy2007effective} of GraphSage output node embedding and test accuracy on ogbn-products graph. We conjecture that increasing the linear independence of the node embedding subspace helps with improving the classification accuracy. Therefore, initializing embedding vectors to be linearly independent, such as orthogonal, effectively improves the model quality compared to Gaussian initialization. 
}%

\begin{figure}[h]
    \centering
    \includegraphics[width=\columnwidth]{Figures/rank_vs_accu.pdf}
    \caption{Effective rank of GraphSage output node embedding vs. test accuracy on ogbn-products. Lighter markers represent models with smaller TT rank.}
    \label{fig:rank-vs-acc}
\end{figure}
\end{comment}

\begin{table*}
\caption{Accuracy of \textsf{ogbn-products} for different TT ranks, initialization algorithms, and graph partitioning.}
\label{tab:prod-accu}
\resizebox{\textwidth}{!}{%
\begin{tabular}{llcccccccccccccc}
\toprule
         &          & \multicolumn{7}{c}{\textbf{Graphsage}}                       & \multicolumn{7}{c}{\textbf{GAT}}                           \\
\cmidrule(rl){1-2}    \cmidrule(rl){3-9}                                             \cmidrule(rl){10-16}
\multicolumn{2}{c}{\textbf{\# of partitions}} & \num{0} & \num{4} & \num{16} & \num{256} & \num{800} & \num{1600} & \num{3200} & \num{0} & \num{4} & \num{16} & \num{256} & \num{800} & \num{1600} & \num{3200} \\
\cmidrule(rl){1-2}    \cmidrule(rl){3-9}                                             \cmidrule(rl){10-16}
              %& eigen                 & 0.6439 & 0.7046 & 0.735  & 0.7424 & 0.7496 & 0.7522 & \textbf{0.7619} & 0.6439 & 0.745  & 0.7480  & 0.7578 & 0.7535 & 0.7615 & 0.7628 \\
              & decomp-ortho          & 0.6077 & 0.7188 & 0.7382 & 0.7352 & 0.7475 & 0.7508 & 0.7539       & 0.7386 & 0.7416 & 0.7393  & 0.7434 & 0.7505 & 0.7666 & 0.7718 \\
TT-Emb(8)     & ortho-core           & 0.6321 & 0.6941 & 0.7333 & 0.7426 & 0.7517 & 0.7435 & \textbf{0.7599}          & 0.7426 & 0.7541 & 0.7515 &  0.754 & \textbf{0.7746} & 0.7642 & 0.7689 \\
              & Gaussian                & 0.6189 & 0.6859 & 0.7230  & 0.7282 & 0.7188 & 0.7187 & 0.7190         & 0.6951& 0.7147 & 0.7484 & 0.7451 & 0.7641 &0.7554  & 0.7613 \\
\cmidrule(rl){1-2}    \cmidrule(rl){3-9}                                             \cmidrule(rl){10-16}
              %& eigen                 & 0.6803 & 0.7183 & 0.7434 & 0.7378 & 0.7457 & 0.7442 & \textbf{0.7592} & 0.6803 & 0.7629 & 0.7564 & 0.7681 & 0.7756 & 0.7713 & 0.7626 \\
              & decomp-ortho          & 0.6688 & 0.7314 & 0.7489 & 0.7368 & 0.7484 & \textbf{0.7606} & 0.7556          & 0.7544 & 0.7440 & 0.7508 & 0.7512 & \textbf{0.7812} & 0.7726 &  0.7746      \\
TT-Emb(16)    & ortho-core            & 0.6773 & 0.7039 & 0.7362 & 0.7379 & 0.7416 & 0.7434 & 0.7580  & 0.7464 & 0.7610 & 0.7603 & 0.7719 & 0.7803 & 0.7638 & 0.7719 \\
              & Gaussian                & 0.6716 & 0.6958 & 0.7305 & 0.7333 & 0.7385 & 0.7275 & 0.7487 & 0.7100 & 0.7559 & 0.7247 & 0.7531 & 0.7618 & 0.7637 & 0.7583 \\
\cmidrule(rl){1-2}    \cmidrule(rl){3-9}                                             \cmidrule(rl){10-16}
              %& eigen                 & 0.6980 & 0.7170  & 0.7452 & 0.7427 & 0.7456 & 0.7501 & \textbf{0.7609} & 0.6980  & 0.7719 & 0.7671 & 0.7678 & \textbf{0.7900}& 0.7647 & 0.7801   \\
              & decomp-ortho          & 0.6701 & 0.7292  & 0.7414 & 0.7430 & 0.7378 & \textbf{0.7613} & 0.7566     & 0.7285  & 0.7671 & 0.7611 & 0.7747 & 0.7835 & 0.7756 & \textbf{0.7843}  \\
TT-Emb(32)    & ortho-core$^*$            & -      & -      & -      & -      & -      & -      & -      & -      & -      & -      & -      & -      & -      & -      \\
              & Gaussian               & 0.6712 & 0.7123 & 0.7396 & 0.7388 & 0.7485 & 0.7422 & 0.7455 & 0.7464 & 0.7607 & 0.7157 & 0.7719 & 0.7672 & 0.7521 &  0.7636\\
\bottomrule
\end{tabular}%
}
\raggedright
\footnotesize{$*$ The orthogonal initialization applies when factors of embedding dimension is sufficiently larger than TT rank. See Section~\ref{sec:method-ortho} for more details.}
\end{table*}

\subsection{Graph Partitioning}
\label{sec:result-metis}
\bcomment{
\begin{figure*}
     \centering
     \begin{subfigure}[b]{0.49\textwidth}
         \centering
         \includegraphics[width=\textwidth]{Figures/graphsage-metis.pdf}
         \caption{Graphsage}
         \label{fig:sage-metis}
     \end{subfigure}
     %\hfill
     \begin{subfigure}[b]{0.45\textwidth}
         \centering
         \includegraphics[width=\textwidth]{Figures/gat-metis.pdf}
         \caption{GAT}
         \label{fig:gat-metis}
     \end{subfigure}
     \hfill
        \caption{Accuracy of models trained with different graph partitions and TT ranks.}
        \label{fig:metic-accu}
\end{figure*}
}

To partition the graph, we use METIS, an algorithm based on the multilevel recursive-bisection and multilevel k-way partitioning scheme~\cite{karypis1997metis}.
We reorder the nodes in the graph based on the partition results such that nodes in the same partition are indexed continuously (\cref{sec:method-metis}).
This reordering encourages neighboring nodes to be constructed with shared parameters in their TT cores, thus inheriting a common embedding representation in each level of the partition.

To avoid correlation between node ids and topology from the default node ordering, we randomly shuffle the node orders before training or performing other techniques.
Doing so allows us to study the effects of graph partitioning.
Again, we refer to \cref{tab:prod-accu} when discussing the accuracy of GraphSage and GAT trained with different TT ranks and numbers of partitions.
Using even a small number of partitions helps to improve TT-emb's accuracy significantly.
Partitioning the graph into four clusters improves the accuracy of TT-emb Graphsage by up to \Percent{6.7} and for GAT by up to \Percent{3.8}.

However, since the nodes within a partition are randomly ordered, the locality of node indexing is not guaranteed.
With a coarse partition, TT-emb does not capture the shared property, which results in lower accuracy.
As we increase the number of partitions to \num{3200}, the accuracy of Graphsage improves by an average of \Percent{10.5} and of GAT by an average of \Percent{3.2} with orthogonal initialization.
On the other hand, an overly fine-grained graph partition is unnecessary as TT-emb is robust to the bottom-level node permutation within the METIS partition hierarchy.  

For a fixed TT rank, the test accuracy generally improves as the number of partitions increases.
The minimal number of partitions to reach the best for GraphSage and GAT are \num{3200} and \num{800}, respectively.
With the optimal partition configuration, our design outperforms the other two embedding models using any TT rank comparing the the baselines in \cref{tab:summary-small}. 

Moreover, We observe that graph partition is particularly effective in homophily graphs with small TT ranks as shown in \cref{tab:paper-accu}.
In \textsf{ogbn-papers100M}, compared to using default node ordering, the test accuracy of TT-emb(8) improves by \Percent{6} with \num{6400} partitions, while that improvement decreases to \Percent{3} with TT rank 128.
%Moreover, in GAT, we achieved a higher quality embedding than the node features given by the benchmark, which exceed the model accuracy trained with the original node feature up to 2.7\%. 
%\cy{why GAT accuracy does not always increase as we add more partitions?}

\bcomment{
\begin{table*}[]
\caption{GraphSage test accuracy with different number of partitions and TT ranks. The embedding is initialized with orthogonal vectors.}
\label{tab:sage-accu-part}
\begin{tabular}{llllllllllll}
\toprule
\textbf{\#partitions} & \num{0} & \num{2} & \num{4} & \num{8} & \num{16} & \num{32} & \num{64} & \num{128} & \num{256} & \num{512} & \num{1024} \\
\midrule
TT rank 8             & 0.6740  & 0.6661  & 0.6841  & 0.6927  & 0.7301   & 0.7242   &  0.7333  &  0.7396   & 0.7343    & 0.7433    & 0.7380     \\
TT rank 16            & 0.6865  & 0.7091  & 0.7087  & 0.7134  & 0.7355   & 0.7353   &  0.7357  &  0.7432   & 0.7395    & 0.7341    & 0.7383     \\
\bottomrule
\end{tabular}
\end{table*}
}
%\israt{We should highlight the following experiment. Maybe a separate subsection?}
To further understand the robustness of our model against node permutation, we implemented a recursive graph partitioning scheme based on METIS.
We recursively partition each subgraph returned by METIS into $p_i$ clusters at the $i$-th level.
This scheme simulates the hierarchical partitioning algorithm of METIS while allowing us to precisely control the branching at each level.

\Cref{fig:accu-perm} compares the accuracy of models trained with the same recursive partition strategy:
first partition the graph into 125 parts, then divide each part into 140 partitions.
These two numbers are chosen from the first two TT-core dimensions $m_1 = 125$, $m_2 = 140$, and $m_3 = 150$.
We permute the ordering of the partition IDs in the following way:
\begin{itemize}
    \item \textbf{No perm}: keep the original partition returned by METIS
    \item \textbf{1st-level perm}: randomly shuffle the ordering of the first level $m_1$ partitions, each roughly containing $m_2m_3$ nodes
    \item \textbf{2nd-level perm}: randomly shuffle the ordering of the second level $m_1m_2$ partitions, each partition contains $m_3$ nodes on average
\end{itemize}
\begin{figure}
    \centering
    \includegraphics[width=0.8\columnwidth]{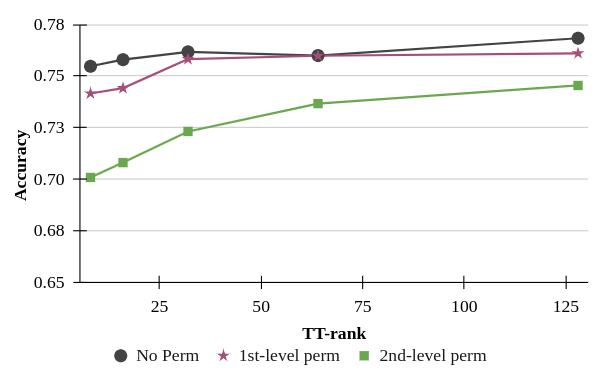}
    \caption{Test accuracy of TT-Emb Graphsage models trained with different graph partition strategies on ogbn-products.}
    \label{fig:accu-perm}
\end{figure}
The accuracy is modestly affected by a 1st-level permutation, while the accuracy decreases by at least \Percent{3} due to 2nd-level permutation.
By \cref{eqn:tt-embedding}, after the METIS reordering, each cluster of $m_2 m_3$ nodes are constructed using the same subtensor in $\mathcal{G}_1$, and every $m_3$ consecutive nodes share the weights from $\mathcal{G}_2$. The 1st-level permutation only interchanges ordering of subtensors $\mathcal{G}_1$, which ensures the graph homophily is preserved.
However, the 2nd-level permutation isolates each group of $m_2$ clusters belonging to a same level-1 partition.
Furthermore, every $m_3$ nodes are constructed with a potentially different subtensor in $\mathcal{G}_1$.
Hence, we observe significant accuracy loss due to lack of homophily.
%Let $i = \sum i_k m_k$, by definition (Equation(\ref{eqn:tt-embedding})), we use the $i_k$-th slice of $\mathcal{G}_k$ to construct the embedding vector. The 1st-level permutation only interchanges ordering of subtensors $\mathcal{G}_1(:, i_1, :, :)$. Each each cluster of $m_2 m_3$ nodes still shares the same weights from $\mathcal{G}_1$ and every $m_3$ consecutive nodes share the weights from $\mathcal{G}_2$. This weight sharing pattern ensures that the graph homophily is preserved. However, the 2nd-level permutation isolates each group of $m_2$ clusters which belong to a same level-1 partition. 
%Take nodes $\{0, 1, 2, \dots, m_2m_3-1\}$ for example, which share the weights $\mathcal{G}_1(:, 0, :, :)$. The 2nd-level permutation shuffle every $m_3$ consecutive nodes $\{lN_3, lN_3+1, \dots,(l+1)N_3\}$ to $\{p_llN_3, p_l(lN_3+1), \dots, p_l(l+1)N_3\}$ for a random integer $1 \le p_l \le N_2$. The $N_2 N_3$ node embeddings will no longer share the tensor weights $\mathcal{G}_1(:, 0, :, :)$, and hence we observe significant accuracy loss due to lack of homophily.

\bcomment{
\begin{table*}[h]
\caption{GraphSage test accuracy with and without randomly permute the first level partition ID.}
\label{tab:sage-accu-perm1}
\begin{tabular}{llllllllllll}
\hline
\textbf{TT rank} & 8 & 16 & 32 & 64 & 128\\ \hline
No perm          & 0.7548 &  0.7580 & 0.7618 & 0.7600 & 0.7684\\
1st-level perm   & 0.7416 &  0.7442 & 0.7583 & 0.7599 & 0.7611\\ 
2nd-level perm   & 0.7009 &  0.7081 & 0.7232 & 0.7367 & 0.7455\\
\hline
\end{tabular}
\end{table*}
}

\begin{table*}[t]
\caption{The original dimensions of embedding tables in OGB datasets and their respective TT decomposition parameters.}
\resizebox{\textwidth}{!}{%
\begin{tabular}{l@{}rrcccrrrrrr@{}}
\toprule
Dataset & \multicolumn{2}{c}{\begin{tabular}[c]{@{}c@{}}Emb. Dimensions\end{tabular}} & \multicolumn{3}{c}{TT-Core Shapes}                  & \multicolumn{3}{c}{\# of TT Parameters} & \multicolumn{3}{c}{Memory Reduction}                                           \\
& \# Nodes                                  & Dim.                                 & $\mathcal{G}_1$ & $\mathcal{G}_2$ & $\mathcal{G}_3$ & $R=16$    & $R=32$    & $R=64$   & $R=16$ & $R=32$ & $R=64$ \\
\cmidrule(rl){1-1} \cmidrule(rl){2-3} \cmidrule(rl){4-6} \cmidrule(rl){7-9} \cmidrule(rl){10-12}
\textsf{ogbn-arxiv} & \num{169 363}                                     & 128                                        & (1, 55, 8, $R$)  & ($R$, 55, 4, $R$)  & ($R$, 56, 4, 1)  & \num{66 944}        & \num{246 528}        & \num{943 616}      & 323 & 88  & 23   \\
\textsf{ogbn-products} & \num{2 449 029}                                      & 100                                        & (1, 125, 4, $R$)  & ($R$, 140, 5, $R$)  & ($R$, 140, 5, 1)  & \num{198 400}      & \num{755 200}        & \num{2 944 000}      & \num{1580} & 415  & 106   \\
\textsf{ogbn-papers100M}\ & \ \num{111 059 956}                                      & 128                                        & (1, 480, 8, $R$)  & ($R$, 500, 4, $R$)  & ($R$, 500, 4, 1)  & \num{605 440}        & \num{2 234 880}        & \num{8 565 760}      & \num{23 479}  & \num{6360}  & \num{1659}   \\
\bottomrule
\end{tabular}%
}
\label{tab:ogb-dim-tt}
\end{table*}

\subsection{Model Compression and Training Time}
\label{sec:results-compress}
Tensor-train compresses the embedding tables significantly.
\Cref{tab:ogb-dim-tt} compares the memory requirement between the full embedding table and TT-embedding across different TT ranks for the three datasets.
On \textsf{ogbn-arxiv}, to achieve the same level of accuracy as Full-Emb, we use a TT rank of at least 64 on GCN and at least 8 on GAT.
This configuration results in a maximum compression ratio of \num{22.9}$\times$ on GCN and \num{1117}$\times$ on GAT.
The memory savings grows as the graph size increases.
As shown in \cref{tab:prod-accu}, the minimum TT rank required to outperform the Full-Emb baseline is 8 for both Graphsage and GAT, which results in a model \num{5762}$\times$ smaller than Full-Emb.
For the largest dataset, \textsf{ogbn-papers100M}, our model achieves a maximum compression ratio of \num{81 362}$\times$ over Full-Emb. 

%\cy{Synthetic data for scalability: size of TT-Emb on billion node graph}

\begin{comment}
\begin{figure}[h]
    \centering
    \begin{subfigure}[b]{0.4\textwidth}
        \centering
        \includegraphics[width=\columnwidth]{Figures/gnn-size.png}
        \caption{Size of TT-Emb}
        \label{fig:gnn-size}
    \end{subfigure}
    \bcomment{
    \begin{subfigure}[b]{0.4\textwidth}
        \centering
        \includegraphics[width=\columnwidth]{Figures/gnn-ratio.png}
        \caption{Compression Ratio of TT-Emb}
        \label{fig:gnn-ratio}
    \end{subfigure}
    }
    \caption{Model size and compression ratio of TT-Emb on OGB datasets.}
    \label{fig:size-ratio}
\end{figure}
\end{comment}

\bcomment{
\begin{table*}[]
\caption{TT compressed embedding table size and memory reduction compared to full embedding.}
\label{tab:datasets-compress}
\begin{tabular}{rrrrrrr}
\toprule
\multicolumn{1}{l}{} & \multicolumn{2}{c}{\textbf{ogbn-arxiv}} & \multicolumn{2}{c}{\textbf{ogbn-products}} & \multicolumn{2}{c}{\textbf{ogbn-papers100M}} \\
TT Rank              & Size (GB)          & Reduction          & Size (GB)            & Reduction           & Size (GB)             & Reduction            \\
\midrule
8                    & 7.76E-05           & 1,117.91           & 2.18E-04             & 5,762.42            & 6.99E-04              & 81,362.61            \\
16                   & 2.68E-04           & 323.83             & 7.94E-04             & 1,580.02            & 2.42E-03              & 23,479.91            \\
32                   & 9.86E-04           & 87.94              & 3.02E-03             & 415.09              & 8.94E-03              & 6,360.82             \\
64                   & 3.77E-03           & 22.97              & 1.18E-02             & 106.48              & 3.43E-02              & 1,659.59             \\
128                  & 1.48E-02           & 5.88               & 4.65E-02             & 26.97               & 1.34E-01              & 424.15               \\
\bottomrule
\end{tabular}
\end{table*}
}

%\subsection{Training speed}
%\label{sec:result-time}
The magnitude of achievable compression can reduce or eliminate the need for CPU-host-to-GPU communication for embeddings, thereby enabling significantly higher GPU utilization.
For the 110~million node dataset, the embeddings need just \GB{0.1}, allowing the storage and training of the full model on a single GPU.
Moreover, in case of multi-GPU training, TT-emb makes it possible to store a copy of embedding table on each GPU, thus reducing communication during training.

Training times are compared \cref{tab:1gpu-time}.
It shows the training time of the Full-Emb model and TT-Emb compressed models for the three datasets using a single GPU.
For the two smaller datasets, the full embedding table can be store on GPU.
By using TT compression, each embedding vector is constructed through two back-to-back GEMMs (general dense matrix-matrix multiplications) instead of pure memory access for lookups.
Therefore, on \textsf{ogbn-arxiv} and \textsf{ogbn-products}, we observe up to \num{33.6}$\times$ slow down with a large TT rank.
However, to achieve the same level of accuracy as the baseline model, one only needs to use a TT rank as small as 8 with sufficiently large partition numbers and our proposed initialization.
Doing so results in 1.28$\times$ and 2.3$\times$ training slowdowns for \textsf{ogbn-arxiv} and \textsf{ogbn-products}, respectively.

For \textsf{ogbn-papers100M}, where the full embedding table does not fit in GPU memory, TT-emb outperforms the baseline with a minimum of 19.3$\times$ speedup in both models using TT rank 8.
\Cref{fig:paper-1gpu-time} breaks down the time spent in each computation phase for \textsf{ogbn-papers100M} training, including forward propagation (\textsf{fwd}), backward propagation (\textsf{bwd}), and parameter update (\textsf{upd}).
The full embedding model requires storing and updating embedding parameters on the CPU, which consist of \Percent{90} of the training time and thus limits the potential of such models on large graphs.
On the other hand, our model fits entirely on the GPU and thus achieves significant speedup in updating model parameters but spend more time in the embedding construction and gradient computation.
%For the tensorized model, the majority of time is spent in forward and backward propagation of TT-Emb, which is the most computationally expensive module.

%\vspace{-0.3cm}
\begin{table}[h]
\caption{Single GPU training time per epoch on ogbn-arxiv, and per batch on ogbn-products and ogbn-papers100M, benchmarked on p3.16xlarge instance.}
\label{tab:1gpu-time}
\resizebox{\columnwidth}{!}{%
\begin{tabular}{lcccccc}
\hline
            & \multicolumn{2}{c}{\textbf{ogbn-arxiv}} & \multicolumn{2}{c}{\textbf{ogbn-products}}              & \multicolumn{2}{c}{\textbf{ogbn-papers100M}} \\
            & GCN                          & GAT      & GraphSage & GAT & GraphSage      & GAT     \\ \hline
Full Emb    & 0.14     & 0.63     & 0.12                        & 0.18                  & 10.34                            & 7.16        \\
TT-emb(8)   & 0.18                         & 0.66     & 0.28                        & 0.28                  & 0.39                             & 0.37  \\
TT-emb(16)  & 0.18                         & 0.67     & 0.33                        & 0.31                  & 0.45                             & 0.45   \\
TT-emb(32)  & 0.20                         & 0.69     & 0.49                        & 0.43                  & 0.51                             & 0.68  \\
TT-emb(64)  & 0.25                         & 0.75     & 1.29                        & 0.99                  & 1.74                             & 1.67  \\
TT-emb(128) & 0.49                         & 0.99     & 4.65                         & 3.37                 & 5.70                             & 5.37 \\ \hline
\end{tabular}
}
\end{table}
%\vspace{-0.2cm}

\begin{figure}
    \centering
    \includegraphics[width=.8\columnwidth]{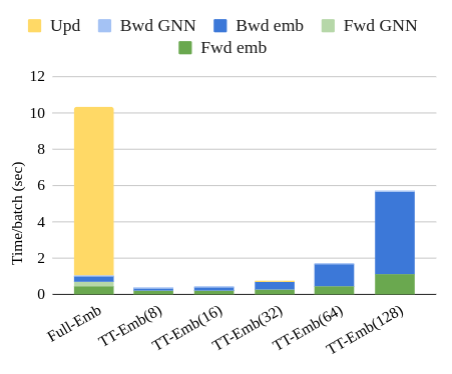}
    \caption{Time spent in each kernel for \textsf{ogbn-papers100M }training.
    Parameter update of Full-Emb is performed on CPU, whereas all TT-Emb operations are on GPU.}
    \label{fig:paper-1gpu-time}
\end{figure}

\begin{comment}
\begin{figure}
    \centering
    \begin{subfigure}[b]{\columnwidth}
        \centering
        \includegraphics[scale=0.4]{Figures/sage-multi-time.png}
        \caption{Graphsage.}
    \end{subfigure}
    \begin{subfigure}[b]{\columnwidth}
        \centering
        \includegraphics[scale=0.4]{Figures/gat_multi_time.png}
        \caption{GAT.}
    \end{subfigure}
    \caption{Multi-GPU training time of TT-Emb models for ogbn-papers100M.}
    \label{fig:multi-time}
\end{figure}
\end{comment}

\begin{figure}[ht]
    \centering
    \includegraphics[scale=0.4]{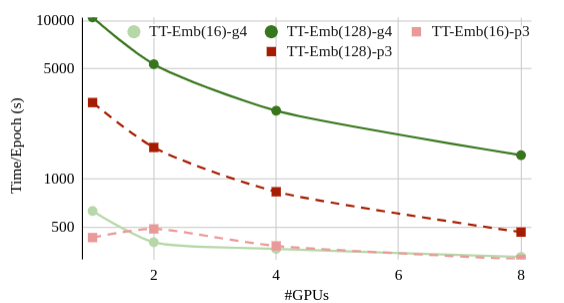}
    \caption{Multi-GPU training time of TT-Emb models for \textsf{ogbn-papers100M} using GraphSage on g4dn.metal instance.}
    \label{fig:multi-time}
\end{figure}

%\vspace{-0.3cm}

We also train the models on multi-GPUs to study the scaling potential of our method for even larger graphs.
\Cref{fig:multi-time} shows the training time per epoch of \textsf{ogbn-papers100M} using all TT-emb models and different number of GPUs.
The green lines represent for the results collected using a \textsf{g4dn.metal} instance, and the red ones are from a \textsf{p3.16xlarge} instance.
Within each color group, the darker shades correspond to model configuration with larger TT ranks.
The small TT rank models benefit less from multi-GPU training. 
Using 8 GPUs, the TT-Emb(8) Graphsage models are less than 2$\times$ faster than single-GPU training on both instances.
Doubling the TT rank increases the computation for TT-emb by 4$\times$.
The embedding computation consist of \Percent{77.2} for TT rank 8, compared to \Percent{95.8} for TT rank 128.
Therefore, TT-emb dominates training time with large TT ranks, and multi-GPU distribution achieves a nearly linear speedup for those models.
%With 8 GPUs, we achieved a speedup of 7.37$\times$ and 6.54$\times$ for Graphsage trained on g4dn.metal and p3.16xlarge, respectively.
%The comparison is similar with GAT, that the training is 7.22$\times$ faster on g4dn.metal and 5.81$\times$ faster on p3.16xlarge using 8 GPUs.

\bcomment{
\begin{table}[H]
\caption{Training time(s) per epoch of GraphSage for ogbn-papers100M using different number of GPUs.}
\label{tab:multigpu-sage}
\begin{tabular}{lllll}
\hline
rank/\#GPUs & 1         & 2         & 4         & 8        \\ \hline
8           & 461.4994  & 367.3915  & 323.6885  &  \\
16          & 632.1466  & 401.7191  & 364.7512  & 324.8754 \\
32          & 1138.6521 & 581.925   & 377.6074  & 330.2817 \\
64          & 3027.1234 & 1557.2169 & 810.3751  & 465.7705 \\
128         & 10474.159 & 5314.4698 & 2707.0591 & 1419.7927 \\ \hline
\end{tabular}
\end{table}

\begin{table}[H]
\caption{Training time(s) per epoch of GAT for ogbn-papers100M using different number of GPUs.}
\label{tab:multigpu-gat}
\begin{tabular}{lllll}
\hline
rank/\#GPUs & 1         & 2        & 4        & 8        \\ \hline
8           & 887.3191  & 499.2087 & 331.0916 & 296.6307 \\
16          & 1016.3493 & 539.7931 & 349.2197 & 309.0176 \\
32          & 1405.0199 & 734.5562 & 416.1655 & 307.7764 \\
64          & 2730.4686 &1429.2423 & 776.3833 & 439.4284 \\
128         & 7989.6843 & 4058.9346 & 2126.9014 &          \\ \hline
\end{tabular}
\end{table}

\begin{table*}[H]
\caption{Detailed training time of GraphSage on \textsf{ogbn-papers100M} using a \textsf{g4dn-16xlarge} instance.}
\label{tab:my-table}
\begin{tabular}{llllllll}
\hline
            & Fwd emb & Fwd GNN & Bwd emb & Bwd GNN & Upd     & Total (sec/batch) & Speedup \\ \hline
Full emb    & 0.7576  & 0.1562  & 0.5086  & 0.0948  & 11.2312 & 12.5158           & -       \\
TT rank 8   & 0.2055  & 0.0349  & 0.1882  & 0.0917  & 0.0036  & 0.5239            & 23.89   \\
TT rank 16  & 0.2375  & 0.0365  & 0.2768  & 0.0871  & 0.0065  & 0.6444            & 19.42   \\
TT rank 32  & 0.3308  & 0.0356  & 0.6143  & 0.082   & 0.0058  & 1.0685            & 11.71   \\
TT rank 64  & 0.6044  & 0.0365  & 1.8882  & 0.1184  & 0.0269  & 2.6744            & 4.68    \\
TT rank 128 & 1.5524  & 0.0356  & 6.9876  & 0.2173  & 0.0024  & 8.7953            & 1.42    \\ \hline
\end{tabular}
\end{table*}
}
\section{Conclusion}
\label{sec:conclusion}

Our most significant results lend strong \emph{quantitative} evidence to an intuitive idea:
one should incorporate knowledge about the global structure of the input graph \emph{explicitly} into the training process.
For example, our proposed algorithm directly uses hierarchical graph structure, discovered via hierarchical graph partitioning, to impose additional structure on the compact parameterization induced by our use of TT decomposition.
Intuitively, this approach encourages homophilic node embeddings;
empirically, we observed instances of improvements in model accuracy by nearly \Percent{5} compared to a state-of-the-art baseline.
It is even possible to learn embeddings \emph{without} using node features while achieving performance competitive with methods that do use such features, which may seem surprising.
Given that many real-world graphs lack such features, this result is especially encouraging.

However, partitioning alone is not enough to produce a good overall method.
A second critical element of our work is an intentional choice of weight initialization where we orthogonalize the TT-Emb weights. 
%and apply an eigenvector decomposition of the graph Laplacian of the input graph.
%These initialization strategies allow the model to capture the positional information at initialization, which significantly improves the model accuracy and convergence.
This initialization strategy significantly improves the model accuracy and convergence.
Our experiments show that by combining the initialization algorithms and graph partitioning, our overall design outperforms the state-of-the-art trainable embedding methods on all node property prediction tasks, with order of magnitude fewer parameters.

Overall, the combination of TT, graph partitioning, and judicious weight-initialization reduces the size of a node embedding by over \num{1659}$\times$ on an industrial graph with 110M nodes, while maintaining or improving accuracy and speeding up training by as much as 5$\times$.
Nevertheless, several important questions remain.
For instance, precisely why orthogonalized embedding for GNNs improves classification accuracy is not fully understood.
And while our methods are effective on homophilic graphs, the performance on heterophilic and heterogenous graphs remains open.

\bibliographystyle{ACM-Reference-Format} 
\bibliography{Tensor-train}

\end{document}